\documentclass[sigconf]{acmart}
\AtBeginDocument{%
  }

\copyrightyear{2026}
\acmYear{2026}
\setcopyright{cc}
\setcctype{by}
\acmConference[WWW '26]{Proceedings of the ACM Web Conference 2026}{April 13--17, 2026}{Dubai, United Arab Emirates}
\acmBooktitle{Proceedings of the ACM Web Conference 2026 (WWW '26), April 13--17, 2026, Dubai, United Arab Emirates}
\acmPrice{}
\acmDOI{10.1145/3774904.3792133}
\acmISBN{979-8-4007-2307-0/2026/04}

\usepackage[utf8]{inputenc} 
\usepackage[T1]{fontenc}    
\usepackage{hyperref}       
 
\usepackage{url}            
\usepackage{booktabs}       
\usepackage{amsfonts}       
\usepackage{nicefrac}       
\usepackage{microtype}      
\usepackage{xcolor}         
\usepackage{enumitem}
\usepackage{amsthm}
\usepackage{amsmath}
\usepackage{bm}

\usepackage{amssymb}
\usepackage{mathtools}
\usepackage{makecell}
\usepackage{multirow}
\usepackage{balance}
\usepackage{booktabs}
\usepackage{wrapfig}
\usepackage{extarrows}
\usepackage{colortbl}
\usepackage{tikz}
\usepackage{pifont}
\usepackage[most]{tcolorbox}
\usepackage{algorithm}
\usepackage{algpseudocode}

\usepackage{graphicx} 
\usepackage{subcaption} 
\usepackage{amsmath} 
\usepackage{lipsum} 

\newcommand{\modelname}{DARK}

\renewcommand\eqref[1]{\textup{(\ref{#1})}}

\usepackage{amsthm}

\theoremstyle{plain} 
\newtheorem{theorem}{Theorem}[section]
\newtheorem{lemma}[theorem]{Lemma}

\theoremstyle{definition} 
\newtheorem{assumption}[theorem]{Assumption}

\newtheorem{definition}[theorem]{Definition}

\theoremstyle{remark}

\newtcolorbox{mathbox}{
  colframe=black,
  colback=white,
  sharp corners,
  boxrule=0.5pt,
  breakable,
  left=1.5pt, right=1.5pt, top=1.5pt, bottom=1.5pt
}

\newcommand{\jx}[1]{\color{black}{#1}}
\begin{document}

\title{Unifying Deductive and Abductive Reasoning in Knowledge Graphs with Masked Diffusion Model}


\author{Yisen Gao}
\email{ygaodi@cse.ust.hk}
\orcid{0009-0006-7564-5171}
\affiliation{%
  \institution{Computer Science and Engineering Hong Kong University of Science and Technology}
  \city{Hong Kong}
  \country{China}
}

\author{Jiaxin Bai}
\orcid{0000-0002-8985-6467}
\email{jbai@connect.ust.hk}
\affiliation{%
  \institution{Computer Science and Engineering Hong Kong University of Science and Technology}
  \city{Hong Kong}
  \country{China}
}

\author{Yi Huang}
\email{yihuang@buaa.edu.cn}
\orcid{0009-0008-6712-9446}
\affiliation{%
 \institution{Computer
Science and Engineering \\Beihang University}
 \city{Beijing}
 \country{China}}

\author{Xingcheng Fu}
\email{fuxc@gxnu.edu.cn}
\orcid{0000-0002-4643-8126}
\affiliation{%
 \institution{Key Lab of Education Blockchain and Intelligent Technology \\Guangxi Normal University}
 \city{Guilin}
 \country{China}}

\author{Qingyun Sun}
\email{sunqy@buaa.edu.cn}
\orcid{0000-0003-1930-3848}
\affiliation{%
 \institution{Computer
Science and Engineering \\Beihang University}
 \city{Beijing}
 \country{China}}

\author{Yangqiu Song}
\authornote{The corresponding author}
\email{yqsong@cse.ust.hk}
\orcid{0000-0002-7818-6090}
\affiliation{%
  \institution{Computer Science and Engineering Hong Kong University of Science and Technology}
  \city{Hong Kong}
  \country{China}
}

\renewcommand{\shortauthors}{Yisen Gao et al.}

\begin{abstract}
  Deductive and abductive reasoning are two critical paradigms for analyzing knowledge graphs, enabling applications from financial query answering to scientific discovery. Deductive reasoning on knowledge graphs usually involves retrieving entities that satisfy a complex logical query, while abductive reasoning generates plausible logical hypotheses from observations. Despite their clear synergistic potential, where deduction can validate hypotheses and abduction can uncover deeper logical patterns, existing methods address them in isolation. To bridge this gap, we propose \textbf{DARK}, a unified framework for \textbf{D}eductive and \textbf{A}bductive \textbf{R}easoning in \textbf{K}nowledge graphs. As a masked diffusion model capable of capturing the bidirectional relationship between queries and conclusions, DARK has two key innovations. First, to better leverage deduction for hypothesis refinement during abductive reasoning, we introduce a self-reflective denoising process that iteratively generates and validates candidate hypotheses against the observed conclusion. Second, to discover richer logical associations, we propose a logic-exploration reinforcement learning approach that simultaneously masks queries and conclusions, enabling the model to explore novel reasoning compositions. Extensive experiments on multiple benchmark knowledge graphs show that DARK achieves competitive performance on both deductive and abductive reasoning tasks, demonstrating the significant benefits of our unified approach.
\end{abstract}



\maketitle

\section{Introduction}
\begin{quote}
\textit{The whole is greater than the sum of its parts.} 
\begin{flushright}
————— Aristotle
\end{flushright}
\end{quote}


{\jx
The long-standing vision of the Semantic Web~\cite{semantic}, as a core pursuit of the web research community, is increasingly realized through large-scale knowledge graphs~\cite{kgsemantic,kgsemantic2}. These structures are now the backbone of the modern web, organizing knowledge to power semantic search~\cite{kgsemantic3}, enhance recommendation systems~\cite{kgrecommend,kgrecommend2}, and enable intelligent assistants~\cite{kgassistant,ngdb} at a global scale. As the web evolves into a decentralized ecosystem of interlinked knowledge graphs~\cite{openkgchain,kgdecentralized}, the ability to reason effectively over this incomplete and noisy data becomes paramount for unlocking the next generation of intelligent Web applications.
}

{\jx 
Within this context, much of the community's focus has been on \textbf{deductive reasoning}~\cite{kgdeduction,kgdeduction2}, which systematically derives conclusions from existing knowledge facts and logical rules. This is crucial for tasks like Complex Query Answering (CQA)~\cite{kgreasoning,kgreasoning2,kgreasoning3}, where users or systems seek specific entities satisfying intricate logical conditions. Existing methods for CQA typically model entities, relations, and logical rules using probabilistic~\cite{betae,probabilistic}, geometric~\cite{multihop,query2box,query2particles}, graph neural network~\cite{ultra,gnnqe} or transformer-based~\cite{enhancing,bai2023complex} approaches, capturing both structural and semantic properties of the graph to enable accurate and efficient query answering.}


On the other hand, \textbf{abductive reasoning}~\cite{abducreview,abduction}, the task of generating hypotheses to best explain observations, has recently gained attention for its potential in more creative and exploratory tasks~\cite{qin2021futureunsupervisedbackpropbaseddecoding, chan2023selfconsistentnarrativepromptsabductive, zhao2024uncommonsensereasoningabductivereasoning}. In the context of the web~\cite{abductiveforweb}, this could mean inferring a user's potential intent from sparse activity logs or discovering new scientific hypotheses from a web of interconnected research data. Building on fundamental works like AbductiveKGR~\cite{akgr}, which extend abductive reasoning tasks to knowledge graphs and propose related generation frameworks, this represents an emerging yet highly promising direction for network-scale discover.


However, existing research~\cite{akgr,multihop,kgdeduction} has largely studied deductive and abductive reasoning separately, overlooking their close interdependence. As illustrated in Fig.~\ref{fig:sub1}, the two paradigms are inherently complementary: abductive reasoning generates hypotheses based on observed entities, which must then be validated through deductive reasoning, while neural-centric deductive reasoning can leverage abductive reasoning to uncover deeper logical structures within the knowledge graph, thereby enhancing model understanding and predictive accuracy. Yet, most approaches treat the two paradigms in isolation, failing to fully harness their combined potential. This motivates the development of a unified framework that integrates both reasoning types, allowing them to mutually reinforce each other and unlock the full power of knowledge graph reasoning {\jx for the web}.

To achieve this unified framework, our key insight is that deductive and abductive reasoning correspond to two directions of reasoning between logical queries and their conclusions. As illustrated in Fig.~\ref{fig:sub2}, the inherent knowledge facts in the knowledge graph serve as the major premises, while each query and its associated conclusion constitute a subset of these facts. Deductive reasoning leverages knowledge and known queries to infer the corresponding conclusions, whereas abductive reasoning uses knowledge and observed conclusions to infer the underlying queries. This intrinsic bidirectionality of reasoning naturally aligns with recent advances in masked diffusion models~\cite{llada,dream7b}, whose inherent bidirectional generation has enabled them to surpass large language models~\cite{gpt4,deepseek} following the autoregressive paradigm on reasoning tasks. Building on this consistency, we are motivated to develop a masked diffusion framework that captures the synergy between deductive and abductive reasoning.

\begin{figure*}[t!]
    \centering
    \begin{subfigure}[t]{0.48\textwidth}
        \centering
        \includegraphics[width=\linewidth]{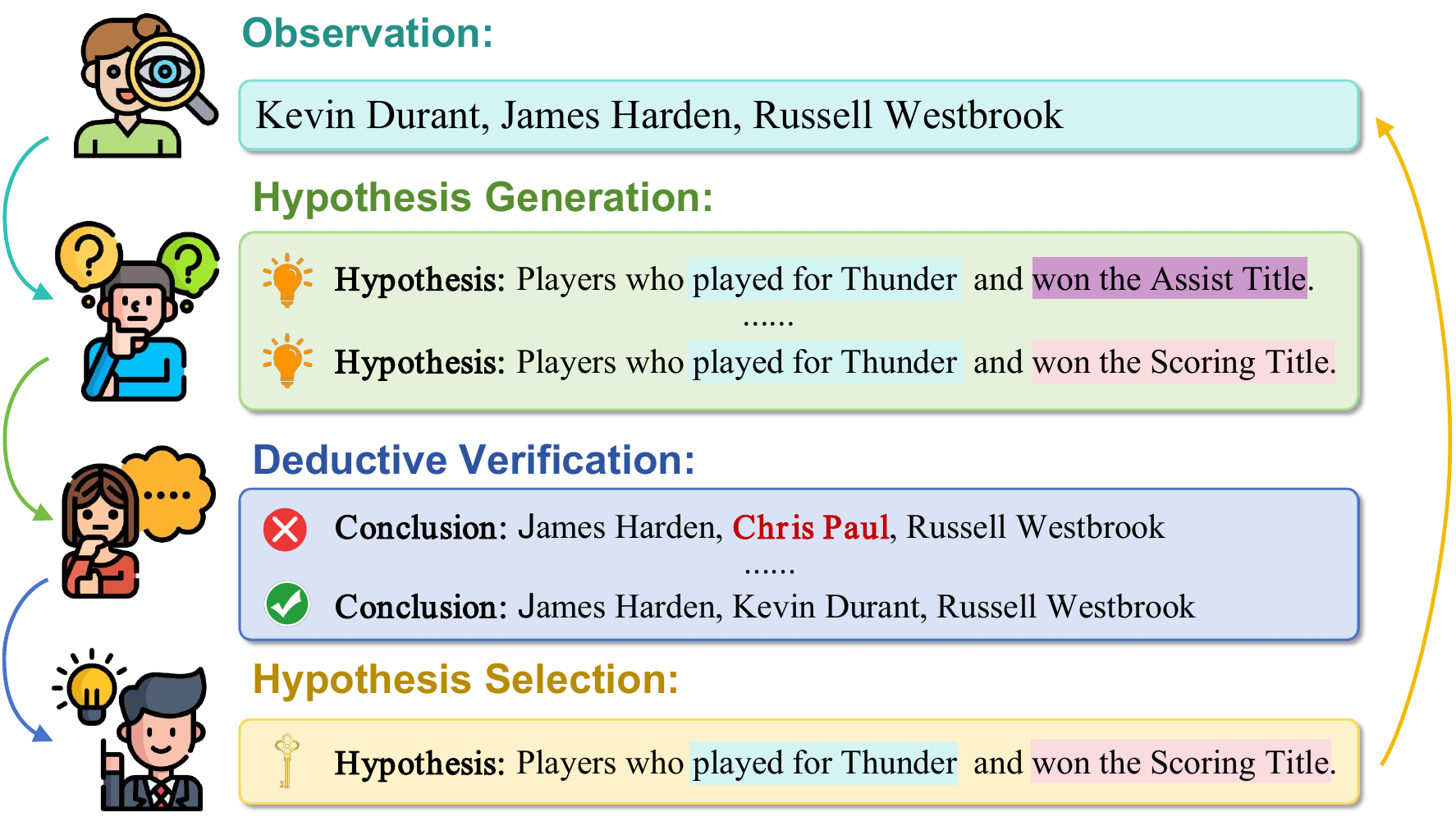}
        \caption{Complementarity between deductive and abductive reasoning}
        \label{fig:sub1}
    \end{subfigure}
    \hfill
    \begin{subfigure}[t]{0.48\textwidth}
        \centering
        \includegraphics[width=\linewidth]{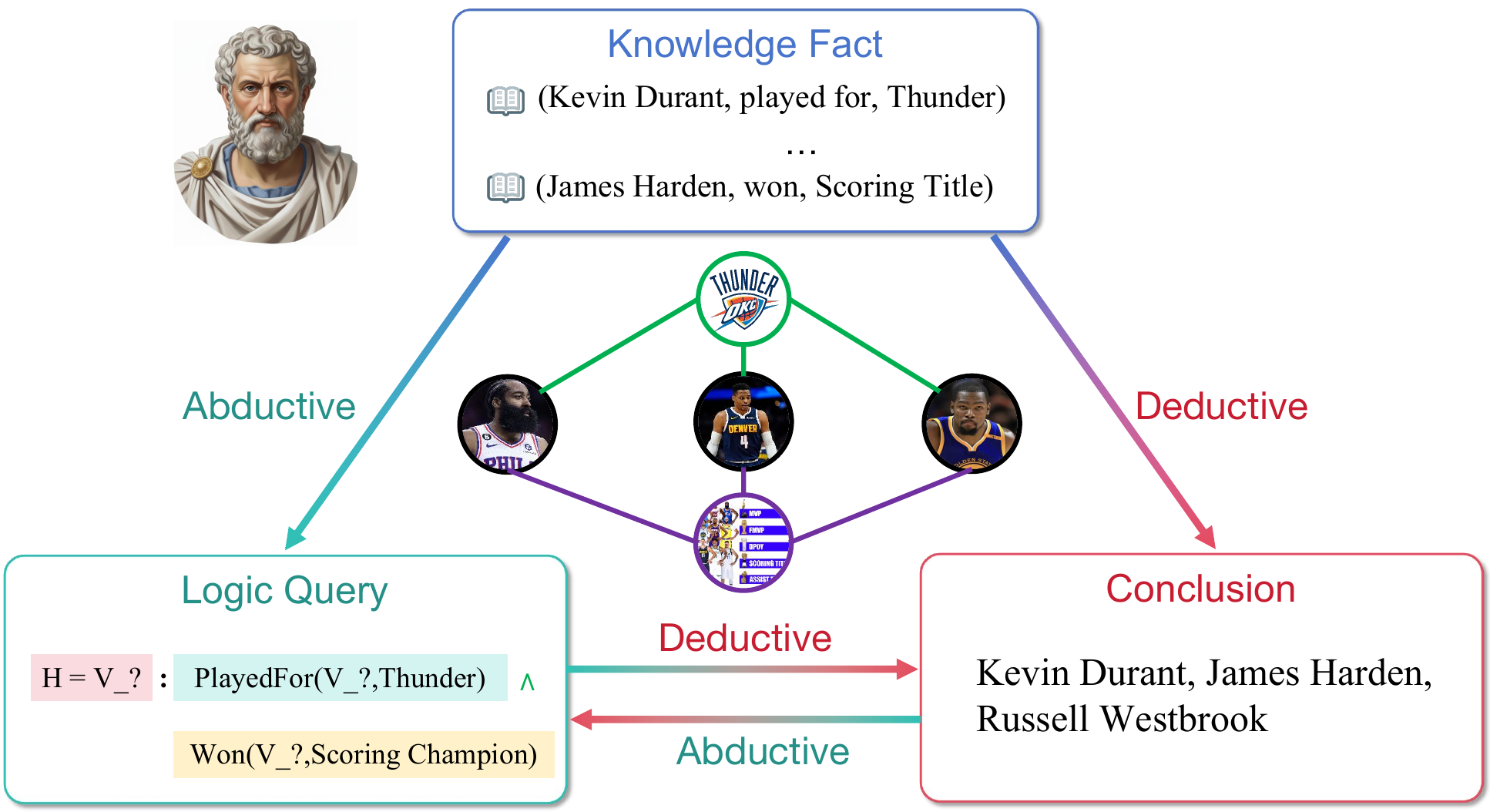}
        \caption{Bidirectional reasoning between queries and conclusions}
        \label{fig:sub2}
    \end{subfigure}
    \caption{The intrinsic connection between deductive reasoning and abductive reasoning.}
    \label{fig:main}
\end{figure*}

However, a naive diffusion model alone is insufficient, as it can only implicitly learn the joint distribution between queries and their corresponding conclusions, without explicitly leveraging the bidirectional synergy between deductive and abductive reasoning. Specifically, it faces two main challenges. First, deductive reasoning is not effectively utilized during the abductive reasoning process, limiting the model’s ability to validate and refine generated hypotheses. Second, the model lacks a thorough exploration of the underlying logical relationships within the knowledge graph through both types of reasoning, which hinders its capacity to capture rich structural and semantic dependencies. 

According to the above analysis, we propose \textbf{DARK}, a masked diffusion model for unifying \textbf{D}eductive and \textbf{A}bductive \textbf{R}easoning in \textbf{K}nowledge graphs. Our approach treats a logical query and its corresponding conclusion as a unified sequence, which is modeled within a masked diffusion framework to capture their joint distribution and enable effective training for downstream reasoning tasks.
To better leverage deductive reasoning for hypothesis verification and refinement during abductive reasoning, we introduce a self-reflective denoising process: at each denoising step, the model generates multiple candidate logical hypotheses and uses deductive reasoning to select the one most consistent with the observed conclusion for the next step.
To address the limited exploration of logical associations, we further propose a logic-exploration reinforcement learning method. By simultaneously masking queries and conclusions, the model can freely explore more plausible logical relationships and uncover richer structural dependencies.
Extensive experiments demonstrate that DARK achieves impressive performance on both abductive and deductive reasoning tasks, validating the effectiveness of our unified framework.
Our contributions can be summarized as follows:

\begin{itemize}[leftmargin=*,itemsep=0pt]
    \item We are the first to unify deductive and abductive reasoning in knowledge graphs by treating them as two directions of the same logical process, and model this using a masked diffusion framework that captures the joint distribution of queries and conclusions.
    \item We introduce a self-reflective denoising mechanism that mirrors human abductive reasoning by generating multiple candidate hypotheses at each step and selecting the one most consistent with the observed conclusion.
    \item To further uncover diverse logical possibilities and semantic associations, we propose a logic-exploration reinforement approach, which simultaneously masks queries and conclusions, enabling flexible exploration of plausible logical relationships under reinforcement learning.
    \item Extensive experiments on three knowledge graph benchmarks demonstrate that DARK achieves impressive performance on both abductive and deductive reasoning tasks, validating the effectiveness and generality of our unified framework.
\end{itemize}

\section{Related Work}

\textbf{Deductive Reasoning in Knowledge Graphs}. In recent years, research on deductive reasoning in knowledge graphs has primarily focused on answering complex logical queries by refining query and answer embeddings with the application of logical operators. Existing approaches can be broadly categorized into several paradigms. Geometric methods~\cite{multihop,query2box,query2particles, bai2023complex, bai2023knowledge, bai2024understanding} embed entities and logical constructs into continuous geometric objects (e.g., boxes, cones), where logical composition is realized as geometric transformations or set operations, thereby enabling interpretable and efficient reasoning.  Probabilistic approaches~\cite{betae,probabilistic} represent queries as distributions over candidate answers, explicitly modeling uncertainty to support multi-step inference. Neural-symbolic frameworks~\cite{faithful} integrate symbolic logical operators with neural network, striking a balance between explicit logical manipulation and flexible representation learning.
Neural methods~\cite{bertqa,gnnqe,ultra,enhancing} directly predict the probability of each candidate entity as an answer, typically in an end-to-end manner using transformers or GNNs to capture multi-hop dependencies and structural patterns.

\textbf{Abductive Reasoning}. Abductive reasoning has been widely studied in both natural language and neuro-symbolic domains. In natural language inference, $\alpha$-NLI~\cite{bhagavatula2020abductivecommonsensereasoning} introduces abductive reasoning into commonsense inference by generating plausible explanations for given observations, with subsequent work extending this paradigm to rare events and more complex logical settings~\cite{qin2021futureunsupervisedbackpropbaseddecoding, kadiķis2022embarrassinglysimpleperformanceprediction, chan2023selfconsistentnarrativepromptsabductive, zhao2024uncommonsensereasoningabductivereasoning}. Benchmarks such as ProofWriter~\cite{tafjord2021proofwritergeneratingimplicationsproofs} evaluate formal abductive reasoning over semi-structured texts with explicit logical relations, while recent studies further explore abductive capabilities of large language models in open-world and abstract reasoning scenarios~\cite{zhong2023chatablabductivelearningnatural, del2023truedetectivedeepabductive, thagard2024chatgptmakeexplanatoryinferences, liu2024incompleteloopinstructioninference, zheng2025logidynamicsunravelingdynamicslogical}. In parallel, neuro-symbolic approaches such as Abductive Learning (ABL)\cite{ABL} integrate perception and symbolic inference, with extensions like ARLC\cite{ABL2} and ABL-Refl~\cite{ABL3} improving context-awareness, error correction, and generalization. Within knowledge graphs, abductive reasoning focuses on generating plausible first-order logical hypotheses given observed entities. Representative methods include AbductiveKGR~\cite{akgr}, which adopts an autoregressive hypothesis generation paradigm, and CtrlHGen~\cite{ctrlhgen}, which introduces controllability to reduce redundant hypotheses. Related studies on rule mining~\cite{rule1,rule2,rule3,rules}, although primarily inductive, also provide useful insights for guiding abductive reasoning over knowledge graphs.
\section{Preliminaries}
\label{sec:preliminary}

We represent a knowledge graph as \( G = (V, R) \), where \( V \) is the set of entities and \( R \) is the set of relation types. Each relation \( r \in R \) is a binary predicate \( r(u, v) \in \{\text{true}, \text{false}\} \), indicating whether the triple \( (u, r, v) \) exists in \( G \).
Under the open-world assumption~\cite{openworld}, the observed graph \( G \) is incomplete and contains only a subset of true facts. Unobserved triples are treated as unknown rather than false, and \( G \) is assumed to be a subgraph of the complete but hidden knowledge graph \( \bar{G} \).

A complex query over knowledge graphs is defined in first-order logic form with logical operators such as existential quantifiers ($\exists$), conjunctions ($\land$), disjunctions ($\lor$), and negations ($\lnot$). The query $Q$ can also be written in disjunctive normal form:
\begin{equation}\begin{aligned}
Q(V_?) & =\exists V_1,\ldots,V_k:e_1\vee\cdots\vee e_n, \\
e_i & =r_{i1}\wedge\cdots\wedge r_{im_i},
\end{aligned}\end{equation}
where ${V_1,\ldots,V_k}$ denotes a subset of $V$.
Each $r_{ij}$ is either of the form $r(u,v)$ or $\neg r(u,v)$, where $u$ and $v$ are drawn from $\{V_1,\ldots,V_k\}$ or correspond to the target variable $V_?$ representing an arbitrary entity in the graph $G$.

\begin{definition}[Complex Query Answering in Knowledge Graphs]
Given a knowledge graph $G$, a first-order logic query $Q$, the goal of complex query answering in the knowledge graph is to find the set of all entities $v_? \in V$ that satisfy the query:
\begin{equation}
Q(v_?) = \text{True}.
\end{equation}
Here, the conclusion of query $Q$ is denoted by $[Q]_{\bar{G}} = \{ v_? \in V \mid Q(v_?) = \text{True} \}$ when executed on the underlying complete graph $\bar{G}$. In this work, we regard the complex query answering task as a form of deductive reasoning on knowledge graphs.

\end{definition}

\begin{definition}[Hypothesis Generation in Knowledge Graphs]
Given a knowledge graph $G=(V,R)$, we define an observation $O$ as the conclusion of a query, represented as a set of entities $O=\{o_1,o_2,\ldots,o_n\}$, where each $o_i \in V$ for all $i \in \{1,\ldots,n\}$. Intuitively, the observation corresponds to the entities that are identified as the answer set of an implicit query in the graph. A hypothesis, in contrast, is expressed as a first-order logic query $Q$ that encodes a possible explanation for why the observation arises from the graph.

The objective of complex hypothesis generation task in knowledge graphs is to identify a hypothesis that serves as the most plausible explanation for the observation $O$, meaning that the conclusion derived from the hypothesis in $G$, denoted as $[Q]_{G}$, is maximally consistent with $O$.
Formally, the goal is to find an optimal hypothesis $Q^*$ that satisfies the following condition:
\begin{equation}
    Q^*=\arg\max_{Q}\operatorname{Jaccard}([Q]_{\bar{G}},O)=\arg\max_{Q}\frac{|[Q]_{\bar{G}}\cap O|}{|[Q]_{\bar{G}}\cup O|}
\end{equation}
Here, the Jaccard Index is used to quantify the similarity between the query result $[Q]_{\bar{G}}$ and the observation $O$, ensuring that the selected $Q^*$ provides the closest possible explanation for the given observation. In this paper, we view the complex logical hypothesis generation task as a problem of abductive reasoning on knowledge graphs, aiming to infer explanatory logical queries that best account for the observed entities.
\end{definition}

\section{Method}
In this section, we propose DARK, a masked diffusion model for unifying deductive and abductive reasoning in knowledge graphs. The framework has been shown in Fig~\ref{fig:frame}.  In Section~\ref{subsec:overview}, we first present the core insights that motivate our approach to unifying these two reasoning paradigms. In Section~\ref{subsec:supervised}, we describe how the masked diffusion model serves as the foundation for this unification and propose a self-reflective sampling strategy, where deduction is employed to refine hypotheses in the abductive reasoning process. In Section~\ref{subsec:rl}, we propose a logic-exploration reinforcement learning method that enables the model to autonomously explore a broader range of logical possibilities, thereby deepening its intrinsic understanding of logic. Finally, we analyze the time and space complexity of the proposed self-reflective sampling method in Appendix B.
\begin{figure*}[htbp]
    \centering
    \includegraphics[width=0.95\textwidth]{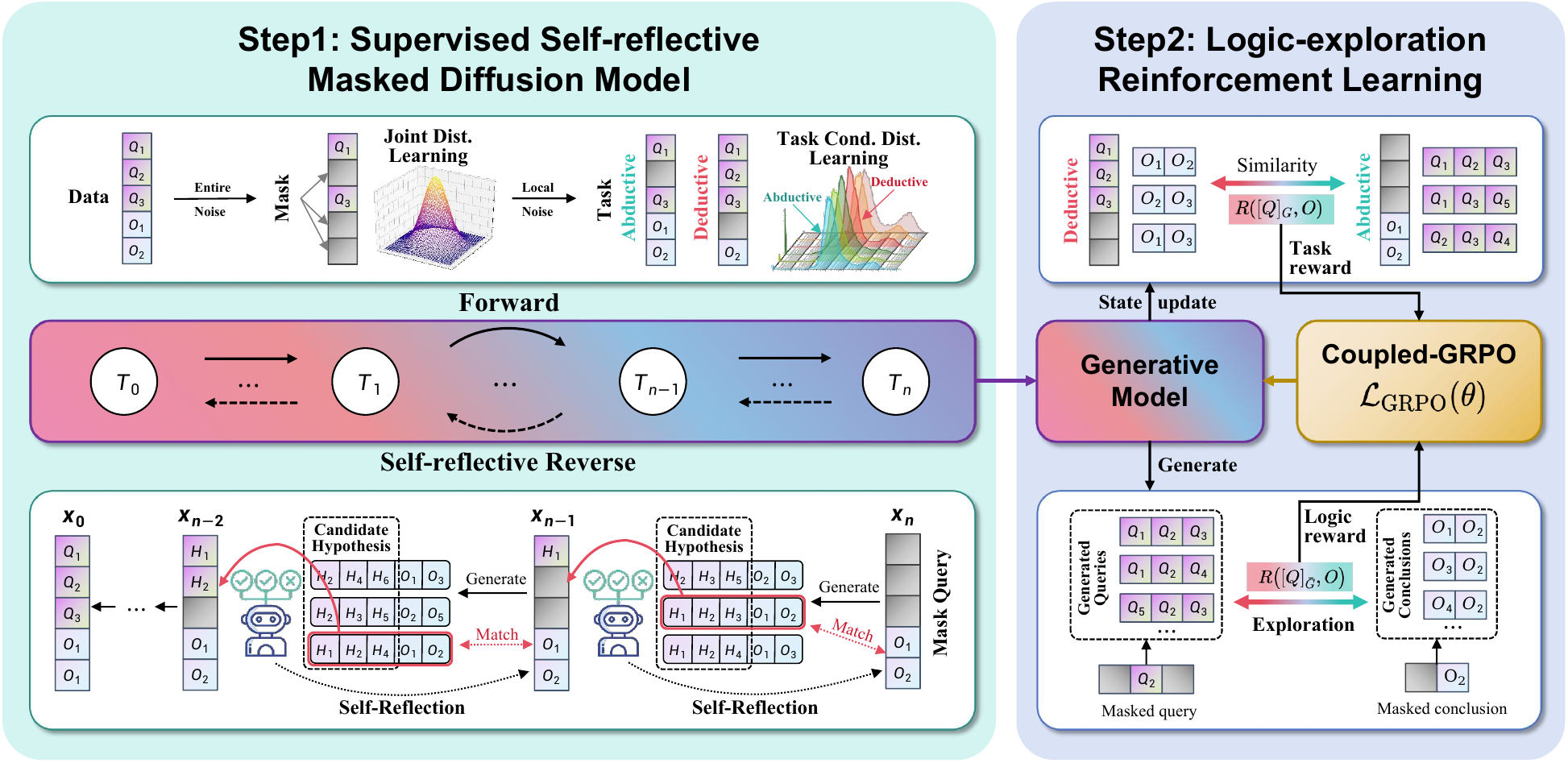}
    \caption{The framework of ~\modelname.}
    \label{fig:frame}
\end{figure*}
\subsection{Overview}
\label{subsec:overview}
Here, we first present an overview of our proposed DARK, which provides a unified framework for both deductive and abductive reasoning in knowledge graphs. Given a query $Q$ and its corresponding answer $O=[Q]_G$ in a knowledge graph $G$, deductive reasoning aims to derive $O$ from $Q$ by leveraging the structured knowledge in $G$, whereas abductive reasoning seeks to infer the most plausible $Q$ that explains an observed $O$. From a high perspective, these two reasoning processes can be framed as conditional probability distributions: deductive reasoning corresponds to $p_G(O| Q)$, and abductive reasoning corresponds to $p_G(Q | O)$. By training a model $\theta$ to capture the intrinsic knowledge of $G$, we simplify these distributions to $p_{\theta}(O \mid Q)$ and $p_{\theta}(Q \mid O)$. This formulation naturally suggests modeling the joint distribution $p_{\theta}(O, Q)$ of queries and answers, providing a unified foundation for reasoning in both directions.

Building on this perspective and motivated by recent advances in bidirectional masked diffusion language models~\cite{llada,dream7b}, we propose a masked diffusion framework that treats a logical query $Q$ and its corresponding conclusion $O$ as a single sequence. By modeling the sequence in both directions, the diffusion framework can capture their joint distribution $p(Q,O)$, naturally enabling both deductive and abductive reasoning within a unified approach.

We now state the main theorem that establishes the generalization bound of applying masked diffusion models to knowledge graph reasoning. The corresponding proofs and details are shown in Appendix~\ref{app:proof}.

\begin{theorem}[Reasoning Convergence of the Masked Diffusion Model]
\label{theorem:reasoning}
Consider a knowledge graph $G$ with its underlying ground-truth graph $\bar{G}$. 
Let $C$ denote the known condition, where $C$ corresponds to the query $Q$ in the deductive setting and to the observation $O$ in the abductive setting. 
Then, under the Assumption~\ref{assump:init},~\ref{assump:earlystop},~\ref{assump:scoreestimate} and ~\ref{assump:bounded}, the reasoning error is bounded as follows:
\begin{equation*}
\begin{aligned}
 E_{C \sim p(C)}\left[D_{\mathrm{KL}}[p_{\theta}(x \mid C)\,\|\,p_{\bar{G}}(x \mid C)]
 \right]\;\le\;&\\ L e^{-T} + \varepsilon_{\mathrm{score}} 
 + L\bigl(T+\log(M\delta^{-1})\bigr)\frac{(T+\log\delta^{-1})^2}{N}
 + \varepsilon_{\mathrm{opa}}.
\end{aligned}
\end{equation*}
where $L$ is the length of the answer $x$, $N$ is the time slice and $T$ is the timestep during denoising. $M$ is a positive number and $\delta$ is a positive number close to 0. 
\end{theorem}

\subsection{Self-reflective Masked Diffusion Model}
\label{subsec:supervised}
Given a query $Q$ and its corresponding answer set $O$, we first convert them into token sequences $X_Q = \{x_{q_1}, x_{q_2}, \ldots, x_{q_m}\}$ and $X_O = \{x_{o_1}, x_{o_2}, \ldots, x_{o_n}\}$, where $m$ and $n$ denote the respective lengths of the sequences. We then create a combined input sequence by joining these token sequences with special tokens:
\begin{equation}
X = \left[ \texttt{[BOS]},\ X_Q,\ \texttt{[SEP]},\ X_O,\ \texttt{[EOS]} \right]
\end{equation}

\subsubsection{Forward Process}
In the masked diffusion framework, the forward process is designed to gradually corrupt the input sequence $X$ by transitioning tokens toward a masked state $\mathbf{M}$. Formally, the transition can be defined as:

\begin{equation}
\begin{aligned}
&p_t(X^t|X^0)=\prod_{i=1}^Lp_t(x_i^t|x_i^0)\\&=\prod_{i=1}^L\mathrm{Cat}(x_i^t;\alpha_t\delta(x_i^0)+(1-\alpha_t)\delta(\mathbf{M})),
\end{aligned}
\end{equation}
where $X^0 = X$ denotes the initial state, $t \in [0, 1]$ is the normalized timestep, $\alpha_t$ is a noise schedule controlling the transition rate, and $\delta$ denotes the Dirac measure on a discrete sample.

Here, we employ distinct masking strategies across two training phases to fully leverage the model’s capacity to capture both joint and conditional probability distribution. In the first phase, the entire sequence $X$ is allowed to transition to the masked state $\mathbf{M}$, enabling the model to learn a comprehensive representation of the underlying joint distribution $p_{\theta}(Q,O)$. In the second phase, we restrict masking to either the query component $X_Q$ or the conclusion component $X_O$ at each iteration, encouraging the model to capture the corresponding conditional distributions $p_{\theta}(Q \mid O)$ and $p_{\theta}(O \mid Q)$. By combining these two training stages, the model can effectively capture the intricate relationships between queries and conclusions, thereby enhancing its performance on downstream reasoning tasks.

\subsubsection{Self-reflective Reverse Process}
In masked diffusion models, the reverse process iteratively reconstructs the original sequence from a masked state $X^1$ by recovering the values of masked tokens.  For abductive reasoning, all query tokens $X_Q^1$ within $X^1$ are masked, while the remaining tokens are known. For deductive reasoning, all observation tokens $X_O^1$ within $X^1$ are masked, with the remaining tokens known.  Once a token has been recovered, it remains unchanged throughout the subsequent steps of the denoising process. For any $0 \leq s < t \leq 1$, the reverse process can be formally expressed as $q(X^s|X^t)=\prod_{i=0}q(x_i^s|X^t)$, where each term $q(x_i^s|X^t)$ is further formulated as:
\begin{equation}
\label{eq:reverse}
q(x^{s}_i|X^t)=
\begin{cases}
\mathrm{Cat}(x^{s}_i;\delta(x_i^t)) & x_i^t\neq\mathbf{M} \\
\mathrm{Cat}\left(x^{s}_i;\frac{(1-\alpha_{s})\delta(\mathbf{M})+(\alpha_{s}-\alpha_t)\tilde{x}^{0}_i}{1-\alpha_t}\right) & x_i^t=\mathbf{M}. 
\end{cases}
\end{equation}
where $\tilde{x}_i^0 = f_{\theta}(X^t)$ represents the estimated value of the original token $x_i^0$ produced by the model $f_{\theta}$, conditioned on the current state $X^t$.

To enhance the integration of deductive reasoning within the abductive reasoning process for hypothesis verification and refinement, we propose a self-reflective reverse process. This method leverages deductive reasoning to iteratively refine the initial hypothesis $\widetilde{X}^0$ at regular intervals of every $k$ denoising steps, enhancing the accuracy of the reconstructed sequence.

Specifically, at each reflective step, we generate a set of $p$ candidate hypotheses, denoted as $\{\hat{X}^{(0,1)}, \hat{X}^{(0,2)}, \dots, \hat{X}^{(0,p)}\}$, where $p$ is the number of sampled candidates. 
For each candidate hypothesis $\hat{X}^{(0,i)}$, we extract its query component $\hat{X}^{(0,i)}_Q$ and employ deductive reasoning to derive the corresponding conclusion $\hat{X}^{(0,i)}_O$ according to $\hat{X}^{(0,i)}_Q$.
We then evaluate the conclusions $\hat{X}^{(0,i)}_O$ against the original observations $X_O$ to identify the candidate hypothesis $\hat{X}^{(0,i)}$ whose conclusion most closely matches $X_O$, selecting it as the refined $\tilde{X}^0$ for the current reflective step. This process iterates in subsequent reverse steps. Formally, the self-reflective process is defined as:
\begin{equation}
    \hat{X}_{O}^{0,i} = f_{\theta}(\hat{X}_{Q}^{(0,i)}, \mathbf{M})
\end{equation}
\begin{equation}
\widetilde{X}_{Q}^0 = \arg\max_{\substack{\hat{X}_{Q}^{(0,i)}}} \text{sim}(\hat{X}_{O}^{(0,i)},X_{O}), \quad i\in[1,2,\ldots,p]
\end{equation} 
where $\text{sim}(\cdot,\cdot)$ represents the Jaccard similarity score.

\subsubsection{Training objective}
The training objective in supervised learning is a refined reformulation of the Evidence Lower Bound (ELBO), designed to optimize the mask predictor \( p_\theta(\cdot | X^t) \) by establishing a computationally tractable variational upper bound on the negative log-likelihood. Formally, this is expressed as a weighted cross-entropy loss:
\begin{equation}
    \mathcal{L}_{t}(X^t) = -\mathbb{E}_{p(X^t | X^0)} \left[ w(t) \sum_{n=1} \mathbf{1}_{[x_n^t = \mathbf{M}]} \log p_\theta(x_n^0 | X^t) \right],
\end{equation}
where  $w(t) \in (0,1]$  is a time-dependent weighting factor derived from a transformation of $\alpha_t$ , and  $\mathbf{1}_{[\mathbf{x}_n^t = \mathbf{M}]}$  is an indicator function that equals 1 if the \( n \)-th token at time \( t \) matches the mask \( \mathbf{M} \), and 0 otherwise.

For the denoising model, we adopt a bidirectional transformer architecture~\cite{dream7b,attention}. A key modification is the removal of the causal mask typically employed in autoregressive models, which allows the model to perform full bidirectional attention and capture contextual information from all tokens in the sequence. Furthermore, following the settings of \cite{dream7b,llada}, we configure the padding token to be the same as EOS token, which enables the model to flexibly generate sequences of variable lengths.

\subsection{Logic-exploration Reinforcement Learning}
\label{subsec:rl}
To explore diverse logic combinations and enhance the understanding of underlying logical structures, we propose a novel logic-exploration reinforcement learning method.
In contrast to conventional reinforement fine-tuning~\cite{akgr} on isolated reasoning tasks, our approach treats logical queries $Q$ and their corresponding conclusions $O$ as an integrated whole.
The model will be encouraged to generate diverse but logically consistent query-conclusion pairs to discover richer logical relationships.
Specifically, the reward function is defined using the Jaccard score, which encourages the model to improve its capability of logical consistency:

\begin{equation}
\begin{aligned}
R([Q]_G, O) &=  \text{Jaccard}([Q]_G, O) =  \frac{|[Q]_G \cap O|}{|[Q]_G \cup O|} ,
\end{aligned}
\end{equation}
where $G$ denotes the observable knowledge graph during training, which serves as a reliable and leakage-free proxy for evaluating reasoning quality.

However, directly generating query–conclusion pairs from sequences of fully masked states is highly challenging, as it may lead the model into a zero advantage dilemma~\cite{zero-delemma}, making it difficult to obtain meaningful exploration outcomes. Moreover, the model can also suffer from a mode-collapse phenomenon, where it repeatedly produces a single query–conclusion pair merely to secure a high reward score. 
To ensure stable training, we randomly mask a subset of tokens from a given query–conclusion pair and use this partially masked instance as the start for denoising, rather than a fully masked state. These partial prompts provide the model with informative guidance during generation, thereby facilitating the production of logically consistent query–conclusion pairs. Furthermore, we incorporate specific task reasoning cases throughout training to further diversify the learning process. This design mitigates the risk of mode collapse, preserves the model’s ability to explore a broader logical space within the knowledge graph, and enhances the effectiveness of downstream reasoning tasks.

To address the inefficiency of conventional Monte Carlo sampling in estimating policy log-probabilities for diffusion models, we adopt the coupled-GRPO~\cite{coupledgrpo} framework. The key component is the precise estimation of token-level log-probabilities.
For each generated query-conclusion pair $(\bar{Q},\bar{O})$, we select $\lambda$  timestep pairs $(t,\hat{t})$ with $t+\hat{t}=1$, and generate two complementary completion masks for each pair. 
In addition, we also incorporate the full completion mask as a reference. 
The log-probability of each token \((q, o)\)  is computed as the average over \(\lambda + 1\) completions:
\begin{equation}
\begin{aligned}
&\log\pi_{\theta}(q,o|q^{t<1},o^{t<1})=\\&\frac{1}{\lambda+1}\left[\sum_{t+\hat{t}=1}^{\lambda}[\mathcal{L}_{t}(q^{t},o^t)+\mathcal{L}_{\hat{t}}(q^{\hat{t}},o^{\hat{t}})+\mathcal{L}_{1}(q^{1},o^1)]\right]
\end{aligned}
\end{equation}

Building upon this estimation, the policy model  $\pi_\theta$ is optimized by maximizing the expected reward over a group of generated query-conclusion pairs $\{(\bar{Q}_1,\bar{O}_1),\ldots,(\bar{Q}_k,\bar{O}_k)\}$, where $k$ is the group size. 
The objective function is defined as:
\begin{equation}
\begin{aligned}
\mathcal{L}_{\mathrm{GRPO}}(\theta) = \mathbb{E}&\{\frac{1}{k}\sum_{i=1}^{k}\frac{1}{|\bar{Q}_i,\bar{O}_{i}|}\sum_{j=1}^{|\bar{Q}_i,\bar{O}_{i}|}\frac{\pi_\theta(q_{ij},o_{ij}|q_{ij}^{t<1},o_{ij}^{t<1})}{\pi_{\theta_\mathrm{old}}(q_{ij},o_{ij}|q_{ij}^{t<1},o_{ij}^{t<1})}\hat{R}'_i
 \\&-\beta D_{\mathrm{KL}}\left[\pi_{\theta}\|\pi_{\mathrm{ref}}\right]\}
\end{aligned}
\end{equation}
where $\hat{R}'_i$ denotes the normalized reward within each group. 
The KL regularization term prevents $\pi_{\theta}$ from deviating excessively from the reference policy $\pi_{\mathrm{ref}}$, with $\beta$ controlling its strength, while gradient clipping is applied to stabilize optimization.

\section{Experiment}
\subsection{Experiment Settings}
\label{sec:expsetting}
\textbf{Dataset}. We evaluate our approach on three widely used knowledge graph datasets: DBpedia50~\cite{dbpedia}, WN18RR~\cite{wn18rr}, and FB15k-237~\cite{fb15k}. Following the common settings in knowledge graph reasoning~\cite{akgr,multihop,query2box}, each dataset is divided into training, validation, and test sets with an 8:1:1 ratio. Under the open-world assumption~\cite{openworld}, we construct the training, validation, and test knowledge graphs, denoted as $G_{\text{train}}$, $G_{\text{valid}}$, and $G_{\text{test}}$, in an incremental manner, such that each subsequent graph encompasses all edges present in its predecessor. This ensures that $G_{\text{valid}}$ contains all edges from $G_{\text{train}}$ and $G_{\text{test}}$ contains all edges from both $G_{\text{train}}$ and $G_{\text{valid}}$.
For query-conclusion pair sampling, we follow prior KG reasoning work~\cite{akgr,query2box} and adopt the 13 predefined logical patterns illustrated in Fig.~\ref{fig:logic_pattern}. Each observed conclusion contains no more than 32 entities. To evaluate generalization, the validation and test sets include entities not seen during training, with the test set covering a larger portion of unseen entities.

\begin{figure}  
 \includegraphics[width=0.9\linewidth]{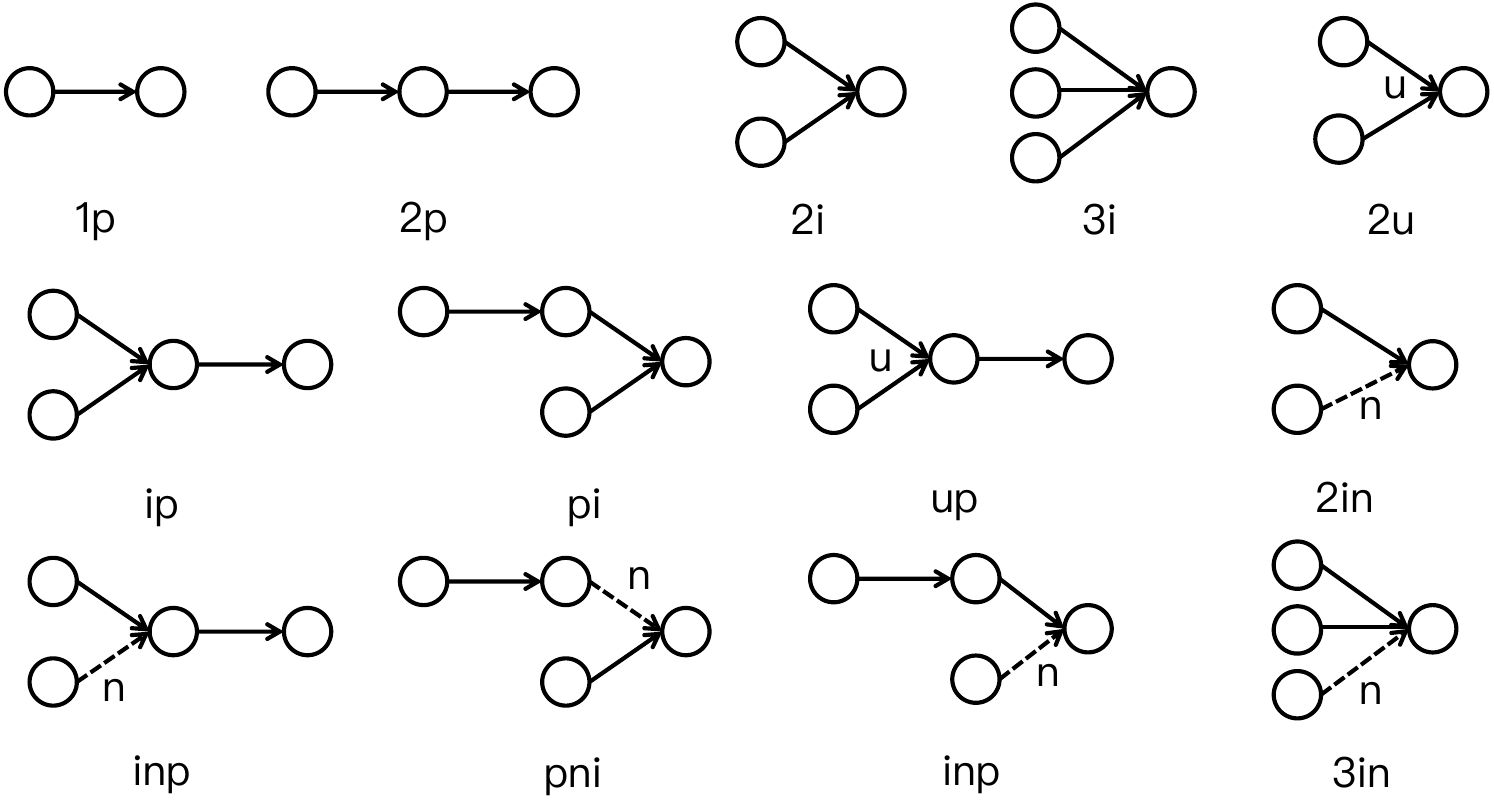}
    \caption{Thirteen predefined logical types.}
    \label{fig:logic_pattern}
\end{figure}

\begin{table*}[h]
\centering
\caption{The average Jaccard score of abductive reasoning among three datasets. (\textbf{Bold}: best; \underline{Underline}: runner-up.)}
\resizebox{\textwidth}{!}{%
\begin{tabular}{l|l|ccccccccccccc|c}
\toprule
Data & Model & 1p & 2p & 2i & 3i & ip & pi & 2u & up & 2in & 3in & pni & pin & inp & ave \\
\midrule
\multirow{5}{*}{FB15k-237} 
 & GPT-4o + 2-hop subgraph &11.2& 5.1&11.6&3.2&2.3&3.0&8.9&1.5&0.7&1.2&1.3&0.3&1.4&4.0 \\
 & Kimi K2 + 2-hop subgraph &6.1 &3.2&7.6&2.2&4.2&2.8&6.0&1.1&2.5&0.9&3.5&1.9&3.0&3.7  \\
 & DeepSeek-V3 + 2-hop subgraph &9.9 &2.4&10.2&2.0&0.9&1.3&8.1&6.6&4.4&0.4&0.5&2.4&1.5&3.9 \\
 & AbductiveKGR & \underline{78.9}&\underline{68.1}&\textbf{65.6}&\textbf{60.5}&\underline{68.3}&\textbf{60.0}&\underline{81.7}&\underline{67.2}&\underline{67.2}&\underline{56.0}&\underline{62.7}&\underline{59.6}&\underline{62.6}&\underline{66.0} \\
\cmidrule{2-16}
 & \modelname 
&\textbf{80.0}&\textbf{72.9}&\underline{65.5}&\underline{52.4}&\textbf{76.5}&\underline{56.9}&\textbf{86.7}&\textbf{69.3}&\textbf{70.0}&\textbf{59.5}&\textbf{62.9}&\textbf{60.2}&\textbf{63.0}&\textbf{67.2} \\

\midrule
\multirow{5}{*}{WN18RR} 
 & GPT-4o + 2-hop subgraph & 8.4 &7.9&14.3&2.9&7.1&7.3&9.3&5.1&1.4&0.8&9.1&7.8&8.2&6.9  \\
 & Kimi K2 + 2-hop subgraph &21.6 &4.6&20.3&9.3&4.5&7.5&10.6&1.9&3.4&1.3&1.7&0.6&2.7&6.2\\
 & DeepSeek-V3 + 2-hop subgraph&11.5 &7.3&14.1&3.2&7.1&2.2&11.4&2.3&0.4&1.4&3.1&2.3&3.8&5.4  \\
 & AbductiveKGR &\textbf{82.1}&\underline{76.0}&\textbf{69.4}&\textbf{69.3}&\textbf{82.7}&\textbf{65.6}&\underline{77.0}&\underline{68.0}&\underline{71.7}&\underline{70.4}&\textbf{72.0}&\underline{67.6}&\underline{72.1}&\underline{72.6}\\
\cmidrule{2-16}
 & \modelname & \underline{72.3}&\textbf{79.8}&\underline{65.1}&\underline{65.7}&\underline{68.7}&\textbf{65.6}&\textbf{84.0}&\textbf{73.2}&\textbf{75.8}&\textbf{70.5}&\underline{66.9}&\textbf{73.0}&\textbf{81.4}&\textbf{73.6}\\
 \midrule
\multirow{5}{*}{DBpedia50} 
 & GPT-4o + 2-hop subgraph &7.5 &0.6&12.3&4.0&2.7&3.1&7.8&3.2&1.0&0.8&2.3&0.7&0.2&3.6  \\
 & Kimi K2 + 2-hop subgraph &5.2&6.5&5.3&3.0&6.0&3.5&3.1&2.9&1.8&3.9&1.9&1.4&2.3&3.6  \\
 & DeepSeek-V3 + 2-hop subgraph &5.3&3.3&6.3&3.1&1.7&0.8&10.0&0.4&2.3&0.1&0.3&1.2&2.8&2.9 \\
 & AbductiveKGR & 
\underline{77.7}&\underline{70.1}&\underline{47.0}&\underline{47.5}&\underline{82.1}&\underline{53.4}&\underline{64.6}&\underline{70.2}&\underline{62.6}&\underline{57.5}&\underline{69.6}&\underline{62.6}&\underline{71.3}&\underline{64.3} \\
\cmidrule{2-16}
 & \modelname & \textbf{83.7}&\textbf{81.7}&\textbf{68.6}&\textbf{57.9}&\textbf{83.8}&\textbf{65.4}&\textbf{78.6}&\textbf{74.5}&\textbf{71.7}&\textbf{62.3}&\textbf{78.3}&\textbf{64.3}&\textbf{71.4}&\textbf{72.5}\\
\bottomrule
\end{tabular}%
}
\label{table:Jaccard}
\end{table*}


\textbf{Implementation Details}\footnote{Our code is available at https://github.com/HKUST-KnowComp/DARK.}. We implement the generative model using a transformer architecture~\cite{dream7b} with 10 layers, a hidden size of 768, and 8 attention heads. For supervised learning, the model is optimized with AdamW using a learning rate of 1e-4, a weight decay of 1e-6, and a warmup of 25 epochs. Training is conducted in two stages: unified training for 200 epochs, followed by training for each downstream task for another 200 epochs. For the reinforcement learning stage, we train the model with a learning rate of 1e-5, a weight decay of 1e-6, and 10 epochs per dataset. All experiments are performed on 4 Nvidia A6000 48GB GPUs, and the reported results are averaged over five independent runs.

\subsection{Abductive Reasoning}

\textbf{Baslines and Metrics}. For abductive reasoning task in knowledge graph, we adopt AbductiveKGR~\cite{akgr} as the baseline, which formulates hypothesis generation in knowledge graphs as an autoregressive generation paradigm. 
We additionally evaluate several advanced LLMs, including GPT-4o~\cite{gpt4}, Kimi K2~\citep{kimik2}, and Deepseek-V3~\citep{deepseek} as the baselines. To address their limited awareness of KG structures, we incorporate 2-hop subgraphs of the observation entities, represented in triple form with the semantic information, into the prompts. 
The quality of the generated hypothesis is predominantly evaluated using the Jaccard score, consistent with its application in abductive reasoning as outlined in Section~\ref{sec:preliminary}. We consider the constructed test graph ${G}_{\text{test}}$ as the latent graph. It is noteworthy that ${G}_{\text{test}}$ incorporates ten percent of edges that were unobserved during the training and validation phases. Formally, given an observation $O$ and a generated hypothesis $H$, we utilize a graph search algorithm to derive the conclusion of $H$ on $G_{\text{test}}$, denoted as $[ H ]_{G_{\text{test}}}$. The Jaccard score is then computed based on this conclusion. The results have been reported in Table~\ref{table:Jaccard}.

First, DARK consistently outperforms AbductiveKGR across all three datasets, establishing new state-of-the-art results. This advantage is particularly pronounced in cases involving negation and logical disjunction, demonstrating the model’s strong capability in understanding and handling complex logical relationships. Beyond this comparison, we also observed that general large language models, when not specifically adapted for the task, fail to achieve competitive performance. We attribute this to three main factors: (i) a limited inherent ability to reason over structured data; (ii) the large size of two-hop subgraphs extracted from observed entities, which leads to long prompts that challenge the models’ capacity to process extended text; and (iii) potential conflicts between the semantic priors embedded in the LLMs and the knowledge graph itself, resulting in misinterpretations. Notably, while large language models can handle relatively simple logical patterns (e.g., 1p, 2i, 2u), their performance degrades significantly as reasoning complexity increases, especially when negation is required to narrow down candidate entities.

\begin{table*}[ht]
\centering
\caption{ The average results of deductive reasoning among three datasets. (\textbf{Bold}: best; \underline{Underline}: runner-up.)}
\begin{tabular}{l|cccc|cccc|cccc}
\toprule
\multirow{2}{*}{\textbf{Model}}  & \multicolumn{4}{c|}{\textbf{FB15k-237}} 
 & \multicolumn{4}{c|}{WN18RR} & \multicolumn{4}{c}{DBpedia50}\\
\cmidrule(lr){2-5} \cmidrule(lr){6-9}\cmidrule(lr){10-13}
 & MRR & H@1 & H@3 & H@10 & MRR & H@1 & H@3 & H@10 & MRR & H@1 & H@3 & H@10 \\
\midrule
BetaE & 24.4 & 13.6 & 27.8 & 47.3 & 25.6 & 20.1 & 29.3 & 35.2 & 22.2 & 18.2 & 23.6 & 30.1  \\
FuzzyQE & 26.6 & 15.3 & 30.4 & 50.4 & 29.8 & 25.0 & 32.6 & 38.4 & 22.3 & 19.2 & 24.5 & 30.9 \\
ULTRAQUERY &\textbf{43.7}  & \textbf{38.3}&\underline{46.5}  &\textbf{54.2}  &26.5 &20.1  &28.3 & 35.6 & 32.1 &30.0  &32.7  &36.8  \\
Query2particle& 34.8& 26.2 & 37.4 & 52.3 & \underline{34.1} &\underline{42.2} &\underline{44.5}& \textbf{47.0}&\underline{35.1} & \underline{33.9} & \underline{35.5} & \underline{39.8} \\
\midrule 
~\modelname& \underline{42.0} & \underline{36.5} & \textbf{47.8} & \underline{53.6} & \textbf{38.7} & \textbf{42.3} & \textbf{45.4} & \underline{46.0} & \textbf{38.9} & \textbf{38.1}& \textbf{42.2} & \textbf{43.6} \\
\bottomrule
\end{tabular}
\label{table:deduction}
\end{table*}

\subsection{Deductive Reasoning}

\textbf{Baselines and Metrics}. For deductive reasoning, we adopt the complex query answering (CQA) task as our evaluation benchmark. We compare our model with several representative state-of-the-art approaches: BetaE~\cite{betae}, a probabilistic embedding model that uses Beta distributions to represent entities and queries, supporting logical operators while modeling uncertainty. FuzzyQE~\cite{fuzzy}, a query embedding framework based on fuzzy logic, which defines logical operators in a principled, learning-free manner for answering first-order logic queries. Query2Particle~\cite{query2particles}, which encodes each query as a set of particle embeddings, enabling reasoning over arbitrary first-order logic queries by retrieving answers from diverse regions of the embedding space. 
UlTRAQUERY~\cite{ultra}, a foundation model based on GNNs for CQA that captures the four fundamental logical relationships, thereby enabling zero-shot query answering across knowledge graphs.
To assess performance on CQA tasks, we employ standard ranking-based metrics: Mean Reciprocal Rank (MRR), which captures the average reciprocal rank of the correct entity, reflecting how highly the model ranks the true answer;  and Hits@$k$, which measures the proportion of queries where the correct answer appears in the top-$k$ predictions, indicating the model’s effectiveness in retrieving correct entities among its most confident outputs.   Here, since our model cannot directly assign a probability to each entity, we set the probability of the entity answered by the model to 1 and the probability of the others that were not answered to 0.  The results have been reported in Table~\ref{table:deduction}.

\begin{figure*}[t]
    \centering
    \begin{minipage}{0.48\textwidth}
        \centering
        \includegraphics[width=\linewidth]{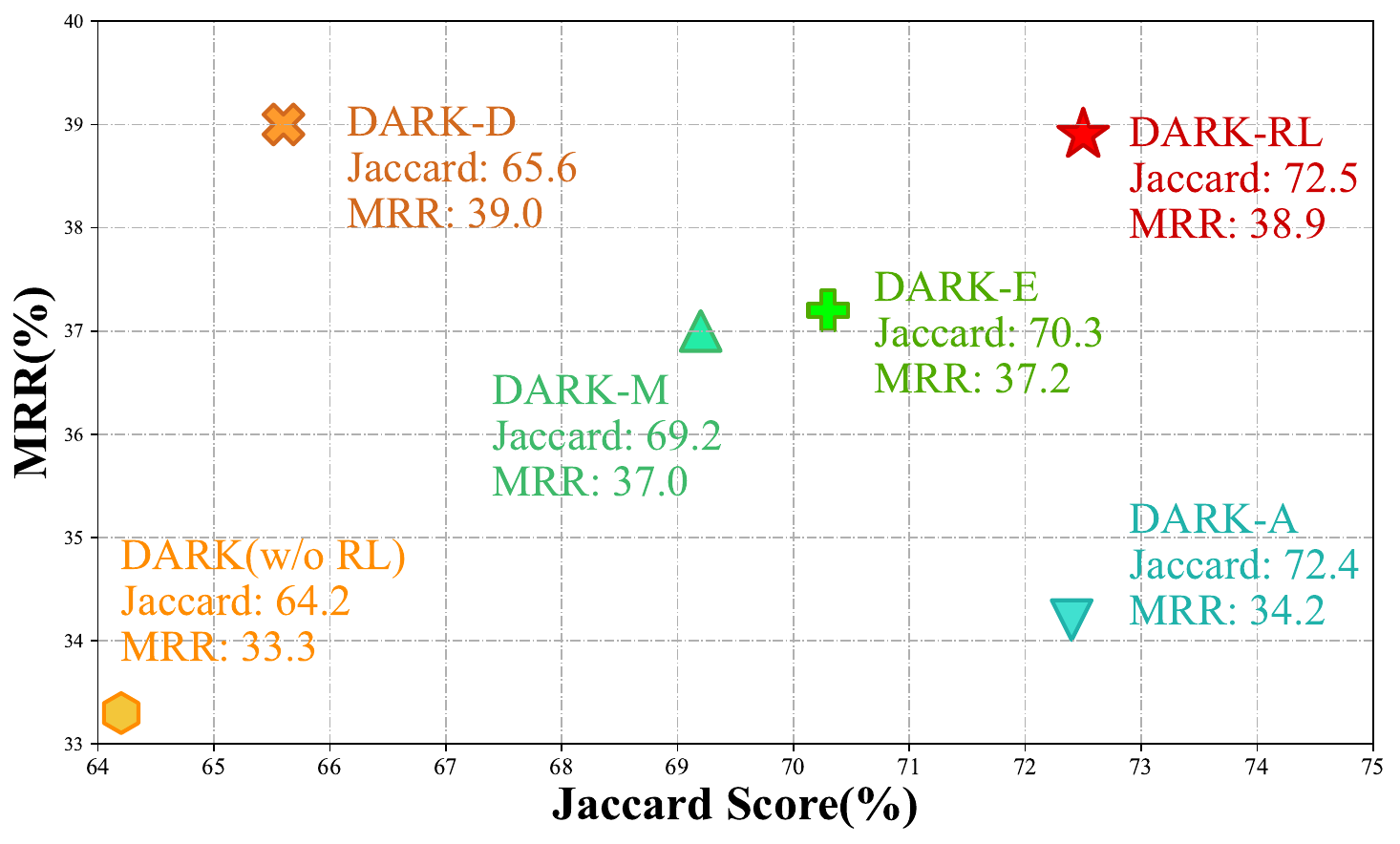}
        \caption{Ablation study for different RL strategy.}
        \label{fig:ablation}
    \end{minipage}
    \hfill
    \begin{minipage}{0.48\textwidth}
        \centering
        \includegraphics[width=\linewidth]{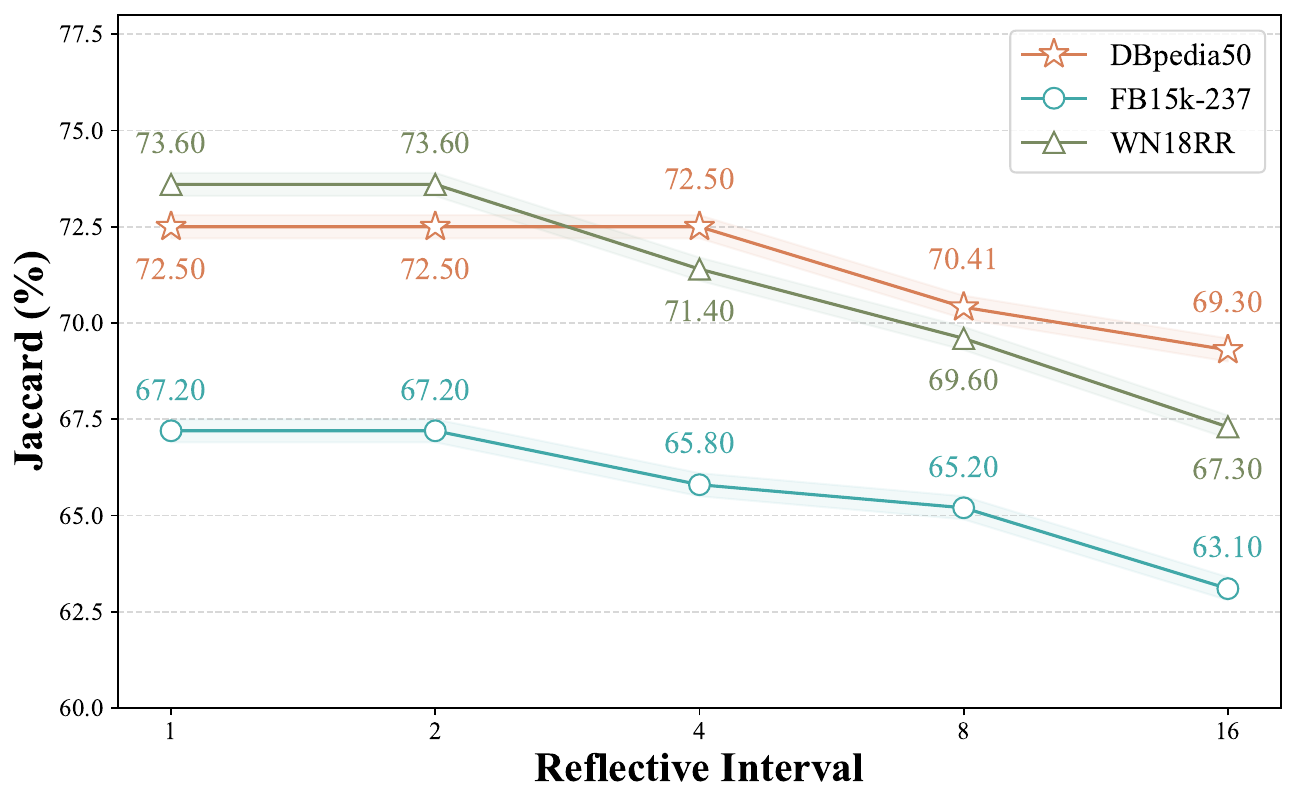}
        \caption{Sensitivity for the reflective interval $k$.}
        \label{fig:sensitivity}
    \end{minipage}
\end{figure*}

Compared with the state-of-the-art methods specifically designed for complex query answering, DARK still demonstrates competitive performance. Notably, the improvement on the DBpedia50 dataset is significant. Furthermore, our method exhibits a clear advantage on the Hits@3 metric, suggesting that the model is effective at generating multiple accurate candidate answers. However, the improvement from Hits@3 to Hits@10 is relatively modest, which we attribute to a limitation of the mask diffusion model in variable-length generation: the model may produce a large number of end-of-sequence tokens prematurely, thereby restricting the generation of additional answer entities.

\subsection{Ablation Study}
In this section, we investigate the effect of different reinforcement learning strategies on model performance. Specifically, we consider several settings: (i) DARK(w/o RL): removing RL entirely; (ii) DARK-A: applying RL only to abductive reasoning; (iii) DARK-D: applying RL only to deductive reasoning; (iv) DARK-M: jointly applying RL to both reasoning tasks in a mixture setting; (v) DARK-E: employing RL with logical exploration; and (vi) DARK-RL: applying RL with both logical exploration and task-specific fine-tuning. To assess the impact of these strategies, we evaluate the Jaccard score and MRR on the DBpedia50 dataset, providing insight into how each configuration influences both abductive and deductive reasoning. All reinforcement learning experiments are uniformly trained for 10 epochs. The results are shown in Fig.~\ref{fig:ablation}.

First, compared with purely supervised learning, incorporating any reinforcement learning (RL) strategy consistently improves performance. Interestingly, applying RL to a single reasoning task also leads to modest gains in the other, suggesting a synergistic relationship between abductive and deductive reasoning: strengthening one type of reasoning implicitly benefits the other. Moreover, introducing logical exploration further enhances the performance of both tasks simultaneously. Finally, unifying task-specific RL with logical exploration yields the most substantial improvements, underscoring their complementary roles in advancing reasoning capability.

\subsection{Sensitivity Study}
Here, we analyze the sensitivity of the reflection interval $k$ in the self-reflective denoising process across three datasets. We evaluate intervals of 1, 2, 4, 8, and 16. The total denoising steps are fixed at 64, corresponding to 64, 32, 16, 8, and 2 reflections, respectively, during sampling. The candidate hypothesis number is 4 for all settings. Results are shown in Fig.~\ref{fig:sensitivity}.

The results show that shorter reflection intervals generally lead to stronger model performance, suggesting that the reflection mechanism effectively guides the model toward selecting more accurate hypotheses for subsequent denoising steps. Furthermore, the performance difference between reflection intervals of 1 and 2 is negligible, and the gap between intervals of 2 and 4 remains relatively small. This suggests that moderately reducing the frequency of reflection does not cause the model to drift from the correct reasoning trajectory. These findings highlight a practical trade-off between efficiency and effectiveness: by appropriately tuning the reflection interval, one can balance computational efficiency with high-quality reflection sampling.

\section{Conclusion}


In this work, we propose DARK, a masked diffusion model that unifies deductive and abductive reasoning over knowledge graphs. By formulating bidirectional reasoning as a masked diffusion process, DARK enables deduction and abduction to interact and mutually refine each other during iterative denoising. The self-reflective denoising mechanism supports iterative hypothesis refinement via deductive validation, while a logic-exploration reinforcement learning strategy promotes more diverse reasoning patterns through joint masking of queries and conclusions. Extensive experiments on multiple benchmarks demonstrate that DARK achieves strong performance and generalization across diverse reasoning scenarios. Overall, this work shows the effectiveness of diffusion-based modeling for logical reasoning, providing a unified and effective framework for knowledge graph reasoning.

However, the main limitation of our approach lies in its reliance on the ids of knowledge graph entities and relationships, lacking more flexible natural language question-answering capabilities. Furthermore, this method relies on a large number of hypothesis-observation pairs for sampling, as well as a long training time. This limits the model's ability to transfer knowledge across different knowledge graphs and its scalability. 

\section{Acknowledgements}
The corresponding author is Yangqiu Song. We owe sincerely thanks to all authors for their valuable efforts and contributions. 
The authors of this paper are supported by the ITSP Platform Research Project (ITS/189/23FP) from ITC of Hong Kong, SAR, China, and the AoE (AoE/E-601/24-N), the RIF (R6021-20) and the GRF (16205322) from RGC of Hong Kong, SAR, China and the National Natural Science Foundation of China (No.62462007 and No.62302023).

\newpage
\bibliographystyle{ACM-Reference-Format}
\balance
\bibliography{ref}

\appendix
\counterwithin{table}{section}
\counterwithin{figure}{section}
\counterwithin{equation}{section}

\section{Proof}
\label{app:proof}
In this section, we support the proof of Theorem~\ref{theorem:reasoning}. We first restate the Theorem here.
\begin{theorem}[Reasoning Convergence of the Masked Diffusion Model]
Consider a knowledge graph $G$ with its underlying ground-truth graph $\bar{G}$. 
Let $C$ denote the known condition, where $C$ corresponds to the query $Q$ in the deductive setting and to the observation $O$ in the abductive setting. 
Then, under the Assumption~\ref{assump:init},~\ref{assump:earlystop},~\ref{assump:scoreestimate} and ~\ref{assump:bounded}, the reasoning error is bounded as follows:
\begin{equation*}
\begin{aligned}
 E_{C \sim p(C)}\left[D_{\mathrm{KL}}[p_{\theta}(x \mid C)\,\|\,p_{\bar{G}}(x \mid C)]
 \right]\;\le\;&\\ L e^{-T} + \varepsilon_{\mathrm{score}} 
 + L\bigl(T+\log(M\delta^{-1})\bigr)\frac{(T+\log\delta^{-1})^2}{N}
 + \varepsilon_{\mathrm{opa}}.
\end{aligned}
\end{equation*}
where $L$ is the length of answer $x$, $N$ is the time slice and $T$ is the timestep during denoising. $M$ is a positive number and $\delta$ is a positive number close to 0. 
\end{theorem}

\begin{proof}

For a more convenient representation, we define $q$ as the forward transition process, $Q$ as the state probability transition matrix, and $\bar{q}$ as the reverse process.

First, by the triangle inequality of the KL divergence, we can obtain
\begin{equation}
\begin{aligned}
&\quad D_{\mathrm{KL}}\!\left[p_{\theta}(x \mid C)\,\|\,p_{\bar{G}}(x \mid C)\right]\\
&\le D_{\mathrm{KL}}\!\left[p_{\theta}(x \mid C)\,\|\,p_G(x \mid C)\right]
    + D_{\mathrm{KL}}\!\left[p_G(x \mid C)\,\|\,p_{\bar{G}}(x \mid C)\right] \\
&= D_{\mathrm{KL}}\!\left[p_{\theta}(x \mid C)\,\|\,p_G(x \mid C)\right] + \varepsilon_{\mathrm{opa}}.
\end{aligned}
\end{equation}
This result indicates that the prediction error can be decomposed into two parts: 
the error of the diffusion model itself and the additional error $\varepsilon_{\mathrm{opa}}$ introduced by the incompleteness of the knowledge graph under the open-world assumption.

Next, we derive the error $E_{C \sim p(C)}\left[D_{\mathrm{KL}}[p_{\theta}(x \mid C)\,\|\,p_G(x \mid C)]\right]$ of the diffusion model. 
\begin{lemma}
\label{lemma:pro}
Given probability distributions $p(x| C)$, \( q(x | C) \), and a marginal distribution \( p(C) \), where \( q(C) = p(C) \), the following euqation holds:
\begin{equation*}
E_{C \sim p(C)} \left[ D_{\mathrm{KL}}(p(x \mid C) \| q(x \mid C)) \right] = D_{\mathrm{KL}}(p(x, C) \| q(x, C)).
\end{equation*}
\end{lemma}

\begin{proof}
    \begin{equation}
\begin{aligned}
&\quad E_{C\thicksim p(C)}\left[D_{\mathrm{KL}}(p(x|C)\|q(x|C))\right]\\&=E_{C\thicksim p(C)}\left[{E}_{x\thicksim p(x|C)}\left[{log}\frac{p(x|C)}{q(x|C)}\right]\right]
\\&=E_{C\thicksim p(C)}\left[{E}_{x\thicksim p(x|C)}\left[{log}\frac{{p(x| C)}{p(C)}}{{q(x| C)}{p(C)}}\right]\right]
\\&=E_{C\thicksim p(C)}\left[{E}_{x\thicksim p(x|C)}\left[{log}\frac{{p(x| C)}{p(C)}}{{q(x| C)}{q(C)}}\right]\right]
\\&=E_{C\thicksim p(C)}\left[{E}_{x\thicksim p(x|C)}\left[{log}\frac{{p(x, C)}}{{q(x, C)}}\right]\right]
\\&=D_{\mathrm{KL}}(p(x, C) \| q(x, C))
\end{aligned}
\end{equation}

\end{proof}
According to Lemma~\ref{lemma:pro}, our objective can be transformed into studying the joint probability distribution of sequences composed of known conditions and reasoning objectives:
\begin{equation}
    E_{C \sim p(C)}\left[D_{\mathrm{KL}}[p_{\theta}(x \mid C)\,\|\,p_G(x \mid C)]\right]=D_{\mathrm{KL}}[p_{\theta}(x, C)\,\|\,p_G(x , C)]
\end{equation}

Since $x$ and $C$ are concatenated into a unified sequence $X$ during the derivation, we treat them as a single unknown learning sequence $X=[x,C] \in[V]^{d_x+d_c}$ to replace the joint probability distribution, where $V$ is the vocabular size and $d$ denotes the length of tokens.

\begin{assumption}[Initial Distribution]
\label{assump:init}
In practical reasoning tasks, the condition $C$ is typically known. Accordingly, we treat the second half of the initial masked state as given. To mitigate the divergence issue that arises when using the stationary distribution of a fully masked state as the initial distribution in the reverse process with respect to the KL divergence, we instead adopt a non-stationary distribution $p_{\text{init}}$ as the initialization:
\begin{equation*}
\begin{cases}
p^i_{init}=(1-\epsilon_T)\delta_{[\mathrm{MASK}]}+\frac{\epsilon_T}{V-1}\sum_{j\neq[\mathrm{MASK}]}\delta_j,
& \text{ if } i\le d_x \\
 p^i_{init}={\delta}_{X_i} & \text{ if } i>d_x
\end{cases}
\end{equation*}
where $\epsilon_T$ is a small positive contant that vanishes as $T\to\infty$. 
\end{assumption}

\begin{lemma}[The Approximate Error of the Initial State]
\label{lemma:initstate}
Assume $\epsilon_T=e^{-T}$ and the inital state follows the Assumption~\ref{assump:init}, we have:

\begin{equation}
D_{\mathrm{KL}}(q_T||p_{init})\lesssim de^{-T}.
\end{equation}
    
\end{lemma}

\begin{proof}
It is very easy to break down the goal into two parts:
\begin{equation}
\begin{aligned}
&D_{\mathrm{KL}}(q_{T}||p_{init})  =\sum_{x_T\in[V]^{d_x+d_c}}q_T(x_T)\operatorname{log}\frac{q_T(x_T)}{p_{init}(x_T)} \\
 & =\sum_{x_{T}\in[V]^{d_x+d_c}}q_{T}(x_{T})\operatorname{log}q_{T}(x_{T})
 -\sum_{x_{T}\in[V]^{d_x+d_c}}q_{T}(x_{T})\operatorname{log}p_{init}(x_{T}).
\end{aligned}
\end{equation}

Then, for the first part, we have according to the Jensen's inequality:
\begin{equation}\begin{aligned}
& \quad \sum_{x_{T}\in[V]^{d}}q_{T}(x_{T})\operatorname{log}q_{T}(x_{T})
\\&=\sum_{x_{T}\in[V]^{d}}\mathbb{E}_{x_{0}\sim q_{0}}\left[q_{T|0}(x_{T}|x_{0})\right]\log\mathbb{E}_{x_{0}\sim q_{0}}\left[q_{T|0}(x_{T}|x_{0})\right] 
\\&\leq\sum_{x_T\in[V]^d}\mathbb{E}_{x_0\sim q_0}
q_{T|0}(x_T|x_0)\log q_{T|0}(x_T|x_0)
 \\&=\mathbb{E}_{x_0\sim q_0}\left[\sum_{x_T\in[V]^d}q_{T|0}(x_T|x_0)\log q_{T|0}(x_T|x_0)\right] \\
&{=}\sum_{i=1}^d\mathbb{E}_{x_0^i\sim q_0^i}\left[\sum_{x_T^i\in[V]}q_{T|0}^i(x_T^i|x_0^i)\log q_{T|0}^i(x_T^i|x_0^i)\right]
\end{aligned}\end{equation}

where $d = d_x+d_C$, since the condition will not change during the forward process, so we have:
\begin{equation}
\label{eq:zero}
\sum_{i=d_x+1}^d\mathbb
{E}_{x_0^i\sim q_0^i}\left[\sum_{x_T^i\in[V]}q_{T|0}^i(x_T^i|x_0^i)\log q_{T|0}^i(x_T^i|x_0^i)\right]=0
\end{equation}

Then, by taking advantage of the property that the forward process remains unchanged after being masked, we have:
\begin{equation}
q_{T|0}^i(y|x)\log q_{T|0}^i(y|x)=
\begin{cases}
e^{-T}\log e^{-T} & \mathrm{if~}y=x\neq\mathbf{M} \\
(1-e^{-T})\log{(1-e^{-T})} & \mathrm{if~}y\neq x=\mathbf{M} \\
0 & \mathrm{otherwise} 
\end{cases}
\end{equation}

Then we can further simplify the first item: 
\begin{equation}
\label{eq:simplyforfir}
\begin{aligned} 
& \quad \sum_{i=1}^{d_x}\mathbb{E}_{x_0^i\sim q_0^i}\left[\sum_{x_T^i\in[V]}q_{T|0}^i(x_T^i|x_0^i)\log q_{T|0}^i(x_T^i|x_0^i)\right]
\\&\leq\sum_{i=1}^{d_x} \mathbb{E}_{x_0^i\sim q_0^i}[1\{x_0^i\neq[\mathbf{M}]\}]\left(e^{-T}\log e^{-T}+(1-e^{-T})\log(1-e^{-T})\right)
\end{aligned} 
\end{equation}

According to the Taylor expansion $(1-x)\log(1-x)=-x+O(x^2)$ , Eq~\ref{eq:zero} and Eq~\ref{eq:simplyforfir}, we can get that:

\begin{equation}
\label{eq:firstpart}
\begin{aligned}
& \quad \sum_{i=0}^d\mathbb{E}_{x_0^i\sim q_0^i}\left[\sum_{x_T^i\in[V]}q_{T|0}^i(x_T^i|x_0^i)\log q_{T|0}^i(x_T^i|x_0^i)\right]
\\&=\quad \sum_{i=1}^{d_x}\mathbb{E}_{x_0^i\sim q_0^i}\left[\sum_{x_T^i\in[V]}q_{T|0}^i(x_T^i|x_0^i)\log q_{T|0}^i(x_T^i|x_0^i)\right]+0
\\&\leq\sum_{i=1}^{d_x} \mathbb{E}_{x_0^i\sim q_0^i}[1\{x_0^i\neq[\mathbf{M}]\}]\left(e^{-T}\log e^{-T}+(1-e^{-T})\log(1-e^{-T})\right)
\\&=\left(\sum_{i=1}^{d_x}(1-q_{0}^{i}([\mathbf{M}]))\right)e^{-T}(-T-1)+O(d_xe^{-2T})
\end{aligned} 
\end{equation}

For the second part, we can adopt a similar simplified method like Eq~\ref{eq:zero} and Eq~\ref{eq:simplyforfir}:

\begin{equation}
\label{eq:secondpart}
\begin{aligned} 
& \quad \sum_{x_{T}\in[V]^{d}}q_{T}(x_{T})\log p_{init}(x_{T})  
\\ & =\sum_{x_T\in[V]^d}\left[\sum_{x_0\in[V]^d}q_{T|0}(x_T|x_0)q_0(x_0)\right]\log p_{init}(x_T) 
\\ & =\mathbb{E}_{x_0\sim q_0}\mathbb{E}_{x_T\sim q_{T|0}(\cdot|x_0)}[\log p_{init}(x_T)] 
 \\ & =\sum_{i=1}^d\mathbb{E}_{x_0^i\sim q_0^i}\mathbb{E}_{x_T^i\sim q_{T|0}^i(\cdot|x_0^i)}[\log p_{init}^i(x_T^i)] 
\\ & =\sum_{i=1}^{d_x}\mathbb{E}_{x_{0}^{i}\sim q_{0}^{i}}\left[1\{x_{0}^{i}\neq[\mathbf{M}]\}\left((1-e^{-T})\log(1-\epsilon)+e^{-T}\log\left(\frac{\epsilon}{S-1}\right)\right)\right]+0
\\ & =\sum_{i=1}^{d_x}\mathbb{E}_{x_{0}^{i}\sim q_{0}^{i}}\left[1\{x_{0}^{i}\neq[\mathbf{M}]\}\left((1-e^{-T})\log(1-\epsilon)+e^{-T}\log\left(\frac{\epsilon}{S-1}\right)\right)\right] 
\\ & \quad +\sum_{i=1}^{d_x}\mathbb{E}_{x_0^i\sim q_0^i}\left[1\left\{x_0^i=[\mathbf{M}]\right\}\right]\log(1-\epsilon) 
\\ & =\left(\sum_{i=1}^{d_x}(1-q_{0}^{i}([\mathbf{M}]))\right)\left[(e^{-T}-1)\epsilon+e^{-T}\log\left(\frac{\epsilon}{S-1}\right)\right] \\
 & \quad +\left(\sum_{i=1}^{d_x}(1-q_0^i([\mathbf{M}]))\right)\epsilon-d_x\epsilon+O(d_x\epsilon^2)
\end{aligned}
\end{equation}

Then, combine the Eq~\ref{eq:firstpart} and Eq~\ref{eq:secondpart}, we can have:
\begin{equation}
\begin{aligned}
D_\mathrm{KL}(q_{T}||p_{init}) & \lesssim d_x\epsilon+\left(\sum_{i=1}^{d_x}(1-q_{0}^{i}(\mathbf{M}))\right)\cdot \\
 & \left(e^{-T}\left(-T-1\right)-(e^{-T}-1)\epsilon-e^{-T}\log\left(\frac{\epsilon}{S-1}\right)-\epsilon\right) \\
 & \leq d_x\epsilon+d_xe^{-T}\left((-T-1)-\epsilon-\log\left(\frac{\epsilon}{S-1}\right)\right)
\\&\lesssim d_xe^{-T}.
\end{aligned}
\end{equation}
\end{proof}

\begin{assumption}[Early Stop]
\label{assump:earlystop}
To avoid divergence of the scoring function $f$ at $t=0$, we assume that the denoising process terminates at a small time $\delta$ approaching zero, rather than exactly at $t=0$.
\end{assumption}

\begin{assumption}[Score Estimation Error]
\label{assump:scoreestimate}
The estimated score function $\mathcal{L}$ for $f_t$ at time $t$ satisfies
\begin{equation*}
\sum_{k=0}^{N-1} (t_{k+1}-t_k)\,\mathcal{L}(f_{T-t_k}) \;\leq\; \varepsilon_{\text{score}}.
\end{equation*}
where $N$ denotes the number of time slices in the denoising process.
\end{assumption}

\begin{assumption}[Bounded Score Estimate]
\label{assump:bounded}
We assume that the absolute value of the scoring function has an upper bound. Formally, there exists a positive number $M>1$ that satisfies
$|\log f_{T-t_{k}}|\leq\log M,\forall k=1,\ldots,N.$
\end{assumption}

\begin{lemma}[Convergence  of Masked Diffusion~\cite{absorbtheory}]

According to the Assumption~\ref{assump:init}, Assumption~\ref{assump:earlystop},Assumption~\ref{assump:scoreestimate}, Assumption~\ref{assump:bounded} and Lemma~\ref{lemma:initstate}, we have the bound that: 
    
\begin{equation*}
\begin{aligned}
 D_{\mathrm{KL}}[p_{\theta}(x, C)\,\|\,p_G(x , C)]\le\;L e^{-T} + \varepsilon_{\mathrm{score}} 
 &\\+ L\bigl(T+\log(M\delta^{-1})\bigr)\frac{(T+\log\delta^{-1})^2}{N}.
\end{aligned}
\end{equation*}
where $L$ denotes the length of $x$.
\end{lemma}

\begin{proof}
Due to space limitations, we present the core steps for proving the lemma here. For a more detailed introduction, please refer to~\cite{absorbtheory}.

First of all, based on the Assumption~\ref{assump:earlystop}, we can simplify the goal to the following format:
\begin{equation}
\begin{aligned}
 D_{\mathrm{KL}}[p_{\theta}(x, C)\,\|\,p_G(x , C)]&=D_{\mathrm{KL}}[p_{\theta}(X)\,\|\,p_G(X)]
\\&=D_\mathrm{KL}(\bar{q}_{\delta}||q_{\delta})
\end{aligned}
\end{equation}

Then, using the Girsanov change-of-measure technique, we can obtain:
\begin{equation}
\begin{aligned}
  &D_\mathrm{KL}(\bar{q}_{\delta}||q_{\delta})\leq D_\mathrm{KL}(\bar{q}_{T:\delta}||q_{T:\delta}) 
 \\& =D_\mathrm{KL}(\bar{q}_T||q_T)
\\& \quad +\mathbb{E}\left[\int_{\delta}^{T}\left(f_{t}(X_{t})-s_{t}(X_{t})+s_{t}(X_{t})\log\frac{s_{t}(X_{t})}{f_{t}(X_{t})}\right)*Q\mathrm{d}t\right]
\end{aligned}\end{equation}

Since $\bar{q}_T=p_{init}$ and
according to the Llema~\ref{lemma:initstate}, we have:
\begin{equation}
\mathcal{L}_{init}=D_\mathrm{KL}(\bar{q}_T||q_T)\lesssim Le^{-T}.
\end{equation}

By further deriving the equation, we can decouple it into two parts:
\begin{equation}\begin{aligned}
 & \mathbb{E}\left[\int_{\delta}^{T}\left(f_{t}(X_{t})-s_{t}(X_{t})+s_{t}(X_{t})\log\frac{s_{t}(X_{t})}{f_{t}(X_{t})}\right)*Q\mathrm{d}t\right]
 \\& =\sum_{k=0}^{N-1}\int_{t_k}^{t_{k+1}}\mathbb{E}_{X_t}\left[\left(f_{t}(X_{t})-s_{t}(X_{t})+s_{t}(X_{t})\log\frac{s_{t}(X_{t})}{f_{t}(X_{t})}\right)*Q\right]\mathrm{d}t
\end{aligned}\end{equation}
One part is below, denoiting the score estimation loss:
\begin{equation}\begin{aligned}
\mathcal{L}_{score} & =\sum_{k=0}^{N-1}(t_{k+1}-t_k)\cdot 
\\&  \mathbb{E}_{x_{t_k}}\sum\left(f_{t_k}(x_{t_k})+s_{t_k}(x_{t_k})log\frac{s_{t_k}(x_{t_k})}{f_{t_k}(x_{t_k})}-s_{t_k}(x_{t_k}))\right)Q\leq\varepsilon_{score}.
\end{aligned}\end{equation}

According to~\cite{absorbtheory}, the remaining part can be divide into:
\begin{equation}\mathcal{L}_{disc}\lesssim(T+\log\delta^{-1}+\log M)L\frac{(T+\log\delta^{-1})^2}{N}.\end{equation}

Combining these three inequalities, we can prove this lemma.
\end{proof}

\end{proof}

\section{Complexity Analysis}
Here, we analyze the algorithmic complexity of the proposed self-reflective sampling strategy. The denoising process has $T$ steps, with self-reflection every $k$ steps, resulting in $m=T/k$ reflective steps. At each reflective step, $p$ candidate hypotheses are generated and evaluated, with each logical token having a fixed length of $L$. The reflection process includes three sub-operations: (1) generating hypotheses, (2) verifying them via deductive reasoning, and (3) selecting the optimal hypothesis. Each sub-operation requires one model evaluation, so the computational cost per reflection is $\mathcal{O}(3p)$. The $(T-m)$ non-reflective denoising steps have constant cost, yielding $\mathcal{O}(T-m)$. Thus, the total time complexity is $\mathcal{O}(3pm + T - m)$, and the space complexity is $\mathcal{O}(pL)$.

\end{document}